\documentclass{article}

\usepackage{amsmath,amsthm,amssymb,amsfonts}
\usepackage{stmaryrd}
\usepackage{algorithm}
\usepackage{algorithmic}
\usepackage{graphicx}
\usepackage{titlesec}
\usepackage{hyperref}
\usepackage{xspace}
\usepackage{float}
\usepackage{olo}
\usepackage{caption}
\usepackage{epstopdf}
\usepackage{fullpage}

\allowdisplaybreaks

\usepackage{notation}

\newcommand{\minimize}{\argmin\limits_{x \in \Omega}}

\newcommand{\minimizevalue}{\inf\limits_{x \in \Omega}}
\newcommand{\maximizevalue}{\sup\limits_{x \in \Omega}}
\newcommand{\maximizepairvalue}{\sup\limits_{x,y \in \Omega}}

\newcommand{\bregmanadaregpsi}[2]{\mathcal{B}_{\psi} \del{#1,#2}}

\newcommand{\bregmangunc}[2]{\mathcal{B}_{g} \del{#1,#2}}
\newcommand{\bregmanadareg}[2]{\mathcal{B}_{r_{0:t}} \del{#1,#2}}
\newcommand{\bregmanadaregcumalative}[3]{\mathcal{B}_{r_{0:#3}} \del{#1,#2}}
\newcommand{\bregmanadaregsingle}[3]{\mathcal{B}_{r_{#3}} \del{#1,#2}}

\newcommand{\regderivative}[1]{\nabla r_{0:t} \del{#1}}
\newcommand{\psiderivative}[1]{\nabla \psi \del{#1}}

\newcommand{\stronghlambda}{H_{1:t}+\lambda_{1:t}}

\newcommand{\stronghlambdaminus}{H_{1:t}+\lambda_{1:t-1}}

\newcommand{\compositefunc}[2]{f_t \del{#1} + \Psi_t \del{#2}}

\makeatletter
\newcommand{\newreptheorem}[2]{\newtheorem*{rep@#1}{\rep@title} 
\newenvironment{rep#1}[1]{\def\rep@title{#2 \ref*{##1}}\begin{rep@#1}}{\end{rep@#1}}
}
\makeatother

\newreptheorem{lemma}{Lemma}
\newreptheorem{theorem}{Theorem}
\newreptheorem{claim}{Claim}

\title{Improved Optimistic Mirror Descent for Sparsity and Curvature}
\author{Parameswaran Kamalaruban\\
Australian National University and Data61, Canberra, Australia\\
\texttt{kamalaruban.parameswaran@data61.csiro.au}}
\date{}

\begin{document}

\maketitle

\begin{abstract}
Online Convex Optimization plays a key role in large scale machine learning. Early approaches to this problem were conservative, in which the main focus was protection against the worst case scenario. But recently several algorithms have been developed for tightening the regret bounds in easy data instances such as sparsity, predictable sequences, and curved losses. In this work we unify some of these existing techniques to obtain new update rules for the cases when these easy instances occur together. First we analyse an adaptive and optimistic update rule which achieves tighter regret bound when the loss sequence is sparse and predictable. Then we explain an update rule that dynamically adapts to the curvature of the loss function and utilizes the predictable nature of the loss sequence as well. Finally we extend these results to composite losses.
\end{abstract}
\section{Introduction}
\label{introduction}
The Online Convex Optimization (OCO) problem plays a key role in machine learning as it has interesting theoretical implications and important practical applications especially in the large scale setting where computational efficiency is the main concern. \cite{shalev2011online} provides a detailed analysis of the OCO problem setting and discusses several applications of this paradigm - online regression, prediction with expert advice, and online ranking.

Given a convex set $\Omega \subseteq \mathbb{R}^n$ and a set $\mathcal{F}$ of convex functions, the OCO problem can be formulated as a repeated game between a learner and an adversary. At each time step $t \in \sbr{T}$, the learner chooses a point $x_t \in \Omega$, then the adversary reveals the loss function $f_t \in \mathcal{F}$, and the learner suffers a loss of $f_t \del{x_t}$. The learner's goal is to minimize the regret which is given by
\begin{equation*}
	R \del{T} := \sum_{t=1}^{T}{f_t \del{x_t}} - \minimizevalue{\sum_{t=1}^{T}{f_t \del{x}}}.
\end{equation*}

There are two main classes of minimax optimal update rules for the OCO problem, namely Follow The Regularized Leader (FTRL) and Mirror Descent. In this work we consider the latter class. Given a strongly convex function $\psi$ and a learning rate $\eta > 0$, standard mirror descent update is given by
\begin{equation}
	\label{minimx-omd}
	x_{t+1}=\minimizevalue{\eta \tuple{g_t,x_t}+\bregmanadaregpsi{x}{x_t}},
\end{equation}
where $g_t \in \partial f_t \del{x_t}$ and $\bregmanadaregpsi{\cdot}{\cdot}$ is Bregman divergence (formally defined later).

It is well understood that the minimax optimal algorithms achieve a regret bound of $O\del{\sqrt{T}}$, which cannot be improved for arbitrary sequences of convex losses \cite{zinkevich2003online}. But in practice there are several \textit{easy data} instances such as sparsity, predictable sequences and curved losses, in which much tighter regret bounds are achievable. Even though minimax analysis gives robust algorithms, they are overly conservative on \textit{easy data}. Now we consider some of the existing algorithms that automatically adapt to the \textit{easy data} to learn faster while being robust to worst case as well.

\cite{duchi2011adaptive} replaces the single static regularizer in the standard mirror descent update~\ref{minimx-omd} by a data dependent sequence of regularizers. This is a fully adaptive approach as it doesn't require any prior knowledge about the bound on the term given by $\sum_{t=1}^{T}{\norm{g_t}^2}$ to construct the regularizers. Further for a particular choice of regularizer sequence they achieved a regret bound of the form $O\del*{\max\limits_t{\norm{x_t-x^*}_{\infty}} \sum_{i=1}^{n}{\sqrt{\sum_{t=1}^{T}{g_{t,i}^2}}}}$, which is better than the minimax optimal bound when the gradients of the losses are sparse and the prediction space is box-shaped.

\cite{chiang2012online,rakhlin2012online} have shown that an optimistic prediction $\tilde{g}_{t+1}$ of the next gradient $g_{t+1}$ at time $t$ can be used to achieve tighter regret bounds in the case where the loss functions are generated by some predictable process e.g. i.i.d losses with small variance and slowly changing gradients. For the general convex losses, the regret bound of this optimistic approach is $O \del*{\sqrt{\sum_{t=1}^{T}{\norm{g_t - \tilde{g}_t}_*^2}}}$. But this is a non-adaptive approach since one requires knowledge of the upper bound on $\sum_{t=1}^{T}{\norm{g_t - \tilde{g}_t}_*^2}$ to set the optimal value for the learning rate. Instead we can employ the standard doubling trick to obtain similar bound with slightly worst constants.

Online optimization with curved losses (strong-convex, exp-concave, mixable etc.) is easier than linear losses. When the loss functions are uniformly exp-concave or strongly convex, $O\del{\log{T}}$ regret bounds are achieved with appropriate choice of regularizers \cite{hazan2007logarithmic,hazan2007adaptive}. But this bound will become worse when the uniform lower bound on the convexity parameters is much smaller. In that case \cite{hazan2007adaptive} proposed an algorithm that can adapt to the convexity of the loss functions, and achieves $O\del{\sqrt{T}}$ regret bounds for arbitrary convex losses and $O\del{\log{T}}$ for uniformly strong-convex losses.

Even though \cite{mcmahan2014analysis} has shown equivalence between mirror descent and a variant of FTRL (namely FTRL-Prox) algorithms with adaptive regularizers, no such mapping is available between optimistic mirror descent and optimistic FTRL updates. Recently \cite{2015arXiv150905760M} have combined adaptive FTRL and optimistic FTRL updates to achieve tighter regret bounds for sparse and predictable sequences. In section~\ref{improve-sparsity} we extend this unification to obtain adaptive and optimistic mirror descent updates. We obtained a factor of $\sqrt{2}$ improvement in the regret bound compared to that of \cite{2015arXiv150905760M}, because in their regret analysis they could not apply the strong FTRL lemma from \cite{mcmahan2014analysis}.

In section~\ref{improve-curvature} we consider the adaptive and optimistic mirror descent update with strongly convex loss functions. In this case we achieve tighter logarithmic regret bound without a priori knowledge about the lower bound on the strong-convexity parameters, in similar spirit of \cite{hazan2007adaptive}. We also explain a curvature adaptive optimistic algorithm that interpolates the results for general convex losses and strongly-convex losses.

In practice the original convex optimization problem itself can have a regularization term associated with the constraints of the problem and generally it is not preferable to linearize those regularization terms. In section~\ref{extensions} we extend all our results to such composite objectives as well.

The main contributions of this paper are:
\begin{itemize}
	\item An adaptive and optimistic mirror descent update that achieves tighter regret bounds for sparse and predictable sequences (Section~\ref{improve-sparsity}).
	\item Improved optimistic mirror descent algorithm that adapts to the curvature of the loss functions (Section~\ref{improve-curvature}).
	\item Extension of the unified update rules to the composite objectives (Section~\ref{extensions}).   
\end{itemize}

Omitted proofs are given in the appendix.

\section{Notation and Background}
\label{background}
We use the following notation throughout. For $n \in \mathbb{Z}^{+}$, let $[n]:=\{1,...,n\}$. The $i$th element of a vector $x \in \mathbb{R}^n$ is denoted by $x_i \in \mathbb{R}$, and for a time dependent vector $x_t \in \mathbb{R}^n$, the $i$th element is $x_{t,i} \in \mathbb{R}$. The inner product between two vectors $x,y \in \mathbb{R}^n$ is written as $\tuple{x,y}$. The gradient of a differentiable function $f$ at $x \in \mathbb{R}^n$ is denoted by $\nabla f \del{x}$ or $f' \del{x}$. A superscript $T$, $A^T$ denotes transpose of the matrix or vector $A$. Given $x \in \mathbb{R}^n$, $A=\mathrm{diag}\del{x}$ is the $n \times n$ matrix with entries $A_{i,i}=x_i$ , $i \in [n]$ and $A_{i,j}=0$ for $i \neq j$. For a symmetric positive definite matrix $A \in S_{++}^n$, we have that $\forall{x \neq 0}, x^T A x > 0$. If $A-B \in S_{++}^n$, then we write $A \succ B$. The square root of $A \in S_{++}^n$ is the unique matrix $X \in S_{++}^n$ such that $XX=A$ and it is denoted as $A^{\half}$. We use the compressed summation notation $H_{a:b}$ as shorthand for $\sum_{s=a}^{b}{H_s}$, where $H_s$ can be a scalar, vector, matrix, or function. Given a norm $\norm{\cdot}$, its dual norm is defined as follows $\norm{y}_* := \sup\limits_{x: \norm{x} \leq 1}{\tuple{x,y}}$. For a time varying norm $\norm{\cdot}_{\del{t}}$, its dual norm is written as $\norm{\cdot}_{\del{t},*}$. The dual norm of the Mahalanobis norm $\norm{x}_A := \sqrt{x^T A x}$ is given by  $\norm{y}_{A^{-1}} = \sqrt{y^T A^{-1} y}$.

Given a convex set $\Omega \subseteq \mathbb{R}^n$ and a convex function $f:\Omega \rightarrow \mathbb{R}$, $\partial f \del{x}$ denotes the sub-differential of $f$ at $x$ which is defined as $\partial f \del{x}:=\setdel{g:f\del{y} \geq f\del{x} + \tuple{g,y-x}, \, \forall{y \in \Omega}}$. A function $f: \Omega \rightarrow \mathbb{R}$ is $\alpha$-strongly convex with respect to a general norm $\norm{\cdot}$ if for all $x,y \in \Omega$
\begin{equation*}
	f \del{x} ~\geq~ f \del{y} + \tuple{g,x-y} + \frac{\alpha}{2} \norm{x-y}^2, \, g \in \partial f \del{y}.
\end{equation*}
Bregman divergence with respect to a differentiable function $g$ is defined as follows
\begin{equation*}
	\bregmangunc{x}{y} ~:=~ g \del{x} - g \del{y} - \tuple{\nabla g \del{y},x-y}.
\end{equation*}
Observe that the function $g$ is $\alpha$-strongly convex with respect to $\norm{\cdot}$ if and only if for all $x,y \in \Omega$ : $\bregmangunc{x}{y} \geq \frac{\alpha}{2} \norm{x-y}^2$. In this paper we use the following properties of Bregman divergences
\begin{itemize}
	\item Linearity: $\mathcal{B}_{\alpha \psi + \beta \phi}\del{x,y} = \alpha \mathcal{B}_{\psi}\del{x,y} + \beta \mathcal{B}_{\phi}\del{x,y}$.
	\item Generalized triangle inequality: $\bregmanadaregpsi{x}{y} + \bregmanadaregpsi{y}{z} = \bregmanadaregpsi{x}{z} + \tuple{x-y,\nabla \psi \del{z} - \nabla \psi \del{y}}$.
\end{itemize}

The following proposition \cite{srebro2011universality,beck2003mirror} is handy in deriving explicit update rules for mirror descent algorithms that we have presented in this work.
\begin{proposition}
	\label{simple-form-omd}
	Suppose $\psi$ is strictly convex and differentiable, and $y$ satisfies the condition $\nabla \psi \del{y} = \nabla \psi \del{u} - g$. Then
	\begin{equation*}
		\minimize{\setdel*{\tuple{g,x} + \bregmanadaregpsi{x}{u}}} = \minimize{\bregmanadaregpsi{x}{y}}.
	\end{equation*}
\end{proposition}

\section{Adaptive and Optimistic Mirror Descent}
\label{improve-sparsity}
\begin{algorithm}[tb]
	\caption{Adaptive and Optimistic Mirror Descent}
	\label{algo:ada-opt-omd}
	\begin{algorithmic}
		\STATE {\bfseries Input:} regularizers $r_0, r_1 \geq 0$.
		\STATE {\bfseries Initialize:} $x_1, \hat{x}_1 = 0 \in \Omega$.
		\FOR{$t=1$ {\bfseries to} $T$}
		\STATE Predict $\hat{x}_t$, observe $f_t$, and incur loss $f_t \del{\hat{x}_t}$. \STATE Compute $g_t \in \partial f_t \del{\hat{x}_t}$ and $\tilde{g}_{t+1} \del{g_1,...,g_t}$.
		\STATE Construct $r_{t+1}$ s.t. $r_{0:t+1}$ is $1$-strongly convex w.r.t. $\norm{\cdot}_{\del{t+1}}$.
		\STATE Update
		\begin{align}
			x_{t+1}
			~=~&
			\minimize{\tuple{g_t,x} + \bregmanadareg{x}{x_t}},
			\label{ada-omd-opt-eq1}
			\\
			\hat{x}_{t+1}
			~=~&
			\minimize{\tuple{\tilde{g}_{t+1},x} + \bregmanadaregcumalative{x}{x_{t+1}}{t+1}}.
			\label{ada-omd-opt-eq2}
		\end{align}
		\ENDFOR
	\end{algorithmic}
\end{algorithm}

When the sequences are predictable \cite{chiang2012online} proposed that making an optimistic prediction of $x_{t+1}$ at time $t$ itself using the optimistic prediction $\tilde{g}_{t+1} \del{g_1,...,g_t}$ of the next sub-gradient $g_{t+1}$. For the optimistic prediction choices of $\tilde{g}_{t+1} = \frac{1}{t} \sum_{s=1}^{t}{g_s}$ and $\tilde{g}_{t+1}=g_t$, we obtain the variance bound \cite{hazan2010extracting} and the path length bound \cite{chiang2012online} respectively.

Given a $1$-strongly convex function $\psi$, and a learning rate $\eta > 0$, the optimistic mirror descent update can be given by two stage updates as follows
\begin{align*}
	x_{t+1} ~=~& \minimize{\eta\tuple{g_t,x_t}+\bregmanadaregpsi{x}{x_t}} \\
	\hat{x}_{t+1} ~=~& \minimize{\eta\tuple{\tilde{g}_{t+1},x_t}+\bregmanadaregpsi{x}{x_{t+1}}}.
\end{align*}
Adaptive and Optimistic mirror descent update is obtained by replacing the static regularizer $\psi$ by a sequence of data dependent regularizers $r_t$'s, which are chosen such that $r_{0:t}$ is $1$-strongly convex with respect to $\norm{\cdot}_{\del{t}}$. The unified update is given in Algorithm~\ref{algo:ada-opt-omd}. Note that the regularizer $r_{t+1}$ is constructed at time $t$ (based on the data observed only up to time $t$) and is used in the second stage update \eqref{ada-omd-opt-eq2}. Also observe that by setting $\tilde{g}_t=0$ for all $t$ in Algorithm~\ref{algo:ada-opt-omd} we recover a slightly modified adaptive mirror descent update given by $x_{t+1} = \minimize{\tuple{g_t,x_t}+\bregmanadareg{x}{x_t}}$, where $r_t$ can depend only on $g_1,...,g_{t-1}$.

The following lemma is a generalization of Lemma 5 from \cite{chiang2012online} for time varying norms, which gives a bound on the instantaneous linear regret of Algorithm~\ref{algo:ada-opt-omd}. 
\begin{lemma}
	\label{ada-opt-omd-linear-regret-lemma}
	The instantaneous linear regret of Algorithm~\ref{algo:ada-opt-omd} w.r.t. any $x^* \in \Omega$ can be bounded as follows
	\[
		\tuple{\hat{x}_t - x^* , g_t} ~\leq~ \bregmanadareg{x^*}{x_t} - \bregmanadareg{x^*}{x_{t+1}} + \half \norm{g_t-\tilde{g}_t}_{\del{t},*}^2.		
	\]
\end{lemma}
\begin{proof}
	Consider
	\begin{equation}
		\label{ada-opt-omd-linear-regret-lemma-proof1}
		\tuple{g_t,\hat{x}_t-x^*} 
		=
		\tuple{g_t - \tilde{g}_{t},\hat{x}_t-x_{t+1}} + \tuple{\tilde{g}_{t},\hat{x}_t-x_{t+1}}
		+ \tuple{g_t,x_{t+1}-x^*}.
	\end{equation}
	By the fact that $\tuple{a,b} \leq \norm{a} \norm{b}_* \leq \half \norm{a}^2 + \half \norm{b}_*^2$, we have 
	\begin{equation*}
		\tuple{g_t - \tilde{g}_{t},\hat{x}_t-x_{t+1}} ~\leq~ \half \norm{\hat{x}_t-x_{t+1}}_{\del{t}}^2 + \half \norm{g_t - \tilde{g}_{t}}_{\del{t},*}^2.
	\end{equation*}
	The first-order optimality condition \cite{boyd2004convex} for $x^* = \minimize{\tuple{g,x} + \bregmanadaregpsi{x}{y}}$ is given by
	\begin{equation*}
		\tuple{x^* - z , g} ~\leq~ \bregmanadaregpsi{z}{y} - \bregmanadaregpsi{z}{x^*} - \bregmanadaregpsi{x^*}{y}, \forall{z \in \Omega}.
	\end{equation*}
	By applying the above condition for \eqref{ada-omd-opt-eq2} and \eqref{ada-omd-opt-eq1} we have respectively 
	\begin{align*}
		\tuple{\hat{x}_t-x_{t+1} , \tilde{g}_{t}} 
		~\leq~& \bregmanadareg{x_{t+1}}{x_{t}} - \bregmanadareg{x_{t+1}}{\hat{x}_t} 
		- \bregmanadareg{\hat{x}_t}{x_{t}},
		\\
		\tuple{x_{t+1} - x^* , g_t} 
		~\leq~& \bregmanadareg{x^*}{x_t} - \bregmanadareg{x^*}{x_{t+1}} 
		- \bregmanadareg{x_{t+1}}{x_t}.
	\end{align*}
	Thus by \eqref{ada-opt-omd-linear-regret-lemma-proof1} we have
	\begin{align*}
		&
		\tuple{g_t,\hat{x}_t-x^*}
		\\
		~\leq~&
		\half \norm{\hat{x}_t-x_{t+1}}_{\del{t}}^2 + \half \norm{g_t - \tilde{g}_{t}}_{\del{t},*}^2 
		+
		\bregmanadareg{x^*}{x_t} - \bregmanadareg{x^*}{x_{t+1}} - \bregmanadareg{x_{t+1}}{\hat{x}_t}
		\\
		~\leq~&
		\half \norm{g_t - \tilde{g}_{t}}_{\del{t},*}^2 + \bregmanadareg{x^*}{x_t} - \bregmanadareg{x^*}{x_{t+1}}
	\end{align*}
	where the second inequality is due to 1-strong convexity of $r_{0:t}$ w.r.t. $\norm{\cdot}_{\del{t}}$.
\end{proof}

The following lemma is already proven by \cite{chiang2012online} and used in the proof of our Theorem~\ref{ada-opt-omd-strong-cvx-theorem}.
\begin{lemma}
	\label{point-loss-connection}
	For Algorithm~\ref{algo:ada-opt-omd} we have, $\norm{\hat{x}_t - x_{t+1}}_{\del{t}} ~\leq~ \norm{g_t - \tilde{g}_t}_{\del{t},*}$.
\end{lemma}

The following regret bound holds for Algorithm~\ref{algo:ada-opt-omd} with a sequence of general convex functions $f_t$'s:
\begin{theorem}
	\label{ada-omd-optimistic-thm}
	The regret of Algorithm~\ref{algo:ada-opt-omd} w.r.t. any $x^* \in \Omega$ is bounded by 
	\begin{equation*}
		\sum_{t=1}^{T}{f_t \del{\hat{x}_t} - f_t \del{x^*}} ~\leq~ \half \sum_{t=1}^{T}{\norm{g_t - \tilde{g}_{t}}_{\del{t},*}^2} 
		+ \sum_{t=1}^{T}{\bregmanadaregsingle{x^*}{x_{t}}{t}} + \bregmanadaregsingle{x^*}{x_1}{0} - \bregmanadaregcumalative{x^*}{x_{T+1}}{T}.
	\end{equation*}
\end{theorem}
\begin{proof} 
	Consider
	\begin{align*}
		&
		\sum_{t=1}^{T}{f_t \del{\hat{x}_t} - f_t \del{x^*}}
		\\
		~\leq~&
		\sum_{t=1}^{T}{\tuple{g_t,\hat{x}_t-x^*}}
		\\
		~\leq~&
		\sum_{t=1}^{T}{\half \norm{g_t - \tilde{g}_t}_{\del{t},*}^2 + \bregmanadareg{x^*}{x_t} - \bregmanadareg{x^*}{x_{t+1}}},
	\end{align*}
	where the first inequality is due to the convexity of $f_t$ and the second one is due to Lemma~\ref{ada-opt-omd-linear-regret-lemma}. Then the following simplification of the sum of Bregman divergence terms completes the proof.
	\begin{align*}
		& 
		\sum_{t=1}^{T}{\bregmanadareg{x^*}{x_t} - \bregmanadareg{x^*}{x_{t+1}}}
		\\
		~=~&
		\bregmanadaregsingle{x^*}{x_1}{0} - \bregmanadaregcumalative{x^*}{x_{T+1}}{T} + \sum_{t=1}^{T}{\bregmanadaregsingle{x^*}{x_{t}}{t}} 
	\end{align*}
\end{proof}

Now we analyse the performance of Algorithm~\ref{algo:ada-opt-omd} with specific choices of regularizer sequences. First we recover the non-adaptive optimistic mirror descent \cite{chiang2012online} and its regret bound as a corollary of Theorem~\ref{ada-omd-optimistic-thm}.
\begin{corollary}
	\label{ada-omd-optimistic-cor}
	Given $1$-strongly convex (w.r.t. $\norm{\cdot}$) function $\psi$, define $\mathcal{R}_{\text{max}} \del{x^*} := \max_{x \in \Omega}{\bregmanadaregpsi{x^*}{x}} - \min_{x \in \Omega}{\bregmanadaregpsi{x^*}{x}}$. If $r_t$'s are given by $r_0 \del{x} = \frac{1}{\eta} \psi \del{x}$ (for $\eta > 0$) and $r_t \del{x} = 0, ~\forall{t \geq 1}$, then the regret of Algorithm~\ref{algo:ada-opt-omd} w.r.t. any $x^* \in \Omega$ is bounded as follows 
	\begin{equation*}
		\sum_{t=1}^{T}{f_t \del{\hat{x}_t} - f_t \del{x^*}} ~\leq~ \frac{\eta}{2} \sum_{t=1}^{T}{\norm{g_t - \tilde{g}_{t}}_{*}^2} + \frac{1}{\eta} \mathcal{R}_{\text{max}} \del{x^*}.
	\end{equation*}
	Further if $\sum_{t=1}^{T}{\norm{g_t - \tilde{g}_{t}}_{*}^2} \leq Q$, then by choosing $\eta = \sqrt{\frac{2 \mathcal{R}_{\text{max}} \del{x^*}}{Q}}$, we have 
	\begin{equation*}
		\sum_{t=1}^{T}{f_t \del{\hat{x}_t} - f_t \del{x^*}} ~\leq~ \sqrt{2 \mathcal{R}_{\text{max}} \del{x^*} Q}.
	\end{equation*}
\end{corollary}
\begin{proof} 
	For the given choice of regularizers, we have $r_{0:t} \del{x} = \frac{1}{\eta} \psi \del{x}$ and $\bregmanadareg{x}{y} = \frac{1}{\eta} \bregmanadaregpsi{x}{y}$. Since $r_{0:t}$ is $1$-strongly convex w.r.t. $\frac{1}{\sqrt{\eta}} \norm{\cdot}$, we have $\norm{\cdot}_{\del{t}} = \frac{1}{\sqrt{\eta}} \norm{\cdot}$ and $\norm{\cdot}_{\del{t},*} = \sqrt{\eta} \norm{\cdot}_{*}$. Then the corollary directly follows from Theorem~\ref{ada-omd-optimistic-thm}.
\end{proof}
In this non-adaptive case we need to know an upper bound of $\sum_{t=1}^{T}{\norm{g_t - \tilde{g}_{t}}_{*}^2}$ in advance to choose the optimal value for $\eta$. Instead we can employ the standard doubling trick to obtain similar bounds with slightly worst constants. 

By leveraging the techniques from \cite{duchi2011adaptive} we can adaptively construct regularizers based on the observed data. The following corollary describes a regularizer construction scheme for Algorithm~\ref{algo:ada-opt-omd} which is fully adaptive and achieves a regret guarantee that holds at anytime.
\begin{corollary}
	\label{diagonal-matrix-regret}
	Given $\Omega \subseteq \times_{i=1}^{n} \sbr{-R_i,R_i}$, let
	\begin{align}
		G_0 &~=~ 0
		\label{ada-mat-g0}
		\\
		G_1 &~=~ \gamma^2 I \,\text{ s.t. }\, \gamma^2 I \succcurlyeq \del{g_t - \tilde{g}_t} \del{g_t - \tilde{g}_t}^T, \, \forall{t} 
		\label{ada-mat-g1}
		\\
		G_t &~=~ \del{g_{t-1} - \tilde{g}_{t-1}} \del{g_{t-1} - \tilde{g}_{t-1}}^T, \, \forall{t \geq 2} 
		\label{ada-mat-gt}
		\\
		Q_{1:t} &~=~ \text{diag} \del*{\frac{1}{R_1},...,\frac{1}{R_n}} \text{diag} \del*{G_{1:t}}^{\half}.
		\nonumber
	\end{align}
	If $r_t$'s are given by $r_0 \del*{x} = 0$ and $r_t \del*{x} = \frac{1}{2 \sqrt{2}} \norm{x}_{Q_t}^2$, then the regret of Algorithm~\ref{algo:ada-opt-omd} w.r.t. any $x^* \in \Omega$ is bounded by 
	\begin{equation*}
		\sum_{t=1}^{T}{f_t \del{\hat{x}_t} - f_t \del{x^*}} \leq 2 \sqrt{2} \sum_{i=1}^{n}{R_i \sqrt{\gamma^2 + \sum_{t=1}^{T-1}{\del{g_{t,i} - \tilde{g}_{t,i}}^2}}}.
	\end{equation*}
\end{corollary}
\begin{proof} 
	By letting $\eta = \sqrt{2}$ for the given sequence of regularizers, we get $r_{0:t} \del*{x} = \frac{1}{2 \eta} \norm{x}_{Q_{1:t}}^2$. Since $r_{0:t}$ is $1$-strongly convex w.r.t. $\frac{1}{\sqrt{\eta}} \norm{\cdot}_{Q_{1:t}}$, we have $\norm{\cdot}_{\del*{t}} = \frac{1}{\sqrt{\eta}} \norm{\cdot}_{Q_{1:t}}$ and $\norm{\cdot}_{\del*{t},*} = \sqrt{\eta} \norm{\cdot}_{Q_{1:t}^{-1}}$. By using the facts that $\text{diag} \del*{\alpha_1,...,\alpha_n}^{\half} = \text{diag} \del*{\sqrt{\alpha_1},...,\sqrt{\alpha_n}}$ and $\text{diag} \del*{\beta_1,...,\beta_n} \cdot \text{diag} \del*{\gamma_1,...,\gamma_n} = \text{diag} \del*{\beta_1 \gamma_1,...,\beta_n \gamma_n}$, the $(i,i)$-th entry of the diagonal matrix $Q_{1:t}$ can be given as
	\begin{align*}
		\del*{Q_{1:t}}_{ii} =~& \frac{1}{R_i} \sqrt{\text{diag} \del*{\gamma^2 I + \sum_{s=1}^{t-1}{\del{g_s - \tilde{g}_s} \del{g_s - \tilde{g}_s}^T}}_{ii}}
		\\ 
		=~& \frac{1}{R_i} \sqrt{\gamma^2 + \sum_{s=1}^{t-1}{\del{g_{s,i} - \tilde{g}_{s,i}}^2}}.
	\end{align*}
	
	Now by Theorem~\ref{ada-omd-optimistic-thm} the regret bound of Algorithm~\ref{algo:ada-opt-omd} with this choice of regularizer sequence can be given as follows
	\begin{equation*}
		\sum_{t=1}^{T}{f_t \del{\hat{x}_t} - f_t \del{x^*}} ~\leq~ \half \sum_{t=1}^{T}{\norm{g_t - \tilde{g}_{t}}_{\del{t},*}^2} + \sum_{t=1}^{T}{\bregmanadaregsingle{x^*}{x_{t}}{t}}.
	\end{equation*}
	Consider
	\begin{align*}
		& \half \sum_{t=1}^{T}{\norm{g_t - \tilde{g}_{t}}_{\del{t},*}^2} 
		\\
		~=~&
		\half \sum_{t=1}^{T}{\eta \norm{g_t - \tilde{g}_{t}}_{Q_{1:t}^{-1}}^2}
		\\
		~=~&
		\frac{\eta}{2} \sum_{t=1}^{T}{\sum_{i=1}^{n}{\del*{g_{t,i} - \tilde{g}_{t,i}}^2} \del*{Q_{1:t}}_{ii}^{-1}}
		\\
		~=~&
		\frac{\eta}{2} \sum_{t=1}^{T}{\sum_{i=1}^{n}{\del*{g_{t,i} - \tilde{g}_{t,i}}^2} \frac{R_i}{\sqrt{\gamma^2 + \sum_{s=1}^{t-1}{\del{g_{s,i} - \tilde{g}_{s,i}}^2}}}}
		\\
		~\leq~&
		\frac{\eta}{2} \sum_{i=1}^{n}{R_i \sum_{t=1}^{T}{\frac{\del*{g_{t,i} - \tilde{g}_{t,i}}^2}{\sqrt{\sum_{s=1}^{t}{\del{g_{s,i} - \tilde{g}_{s,i}}^2}}}}}
		\\
		~\leq~&
		\eta \sum_{i=1}^{n}{R_i \sqrt{\sum_{t=1}^{T}{\del{g_{t,i} - \tilde{g}_{t,i}}^2}}}
		\\
		~\leq~&
		\eta \sum_{i=1}^{n}{R_i \sqrt{\gamma^2 + \sum_{t=1}^{T-1}{\del{g_{t,i} - \tilde{g}_{t,i}}^2}}},
	\end{align*}
	where the first and third inequalities are due to the fact that $\gamma^2 \geq \del{g_{t,i} - \tilde{g}_{t,i}}^2$ for all $t \in \sbr{T}$, and the second inequality is due to the fact that for any non-negative real numbers $a_1, a_2, ... , a_n$ : $\sum_{i=1}^{n}{\frac{a_i}{\sqrt{\sum_{j=1}^{i}{a_j}}}} \leq 2 \sqrt{\sum_{i=1}^{n}{a_i}}$.
	Also observing that
	\begin{align*}
		& \sum_{t=1}^{T}{\bregmanadaregsingle{x^*}{x_{t}}{t}} 
		\\
		~=~&
		\sum_{t=1}^{T}{\frac{1}{2 \eta} \norm{x^* - x_t}_{Q_t}^2}
		\\
		~=~&
		\frac{1}{2 \eta} \sum_{t=1}^{T}{\sum_{i=1}^{n}{\del*{x_i^* - x_{t,i}}^2 \del*{Q_t}_{ii}}}
		\\
		~\leq~&
		\frac{1}{2 \eta} \sum_{i=1}^{n}{\del*{2 R_i}^2 \sum_{t=1}^{T}{\del*{Q_t}_{ii}}}
		\\
		~=~&
		\frac{2}{\eta} \sum_{i=1}^{n}{R_i^2 \del*{Q_{1:T}}_{ii}}
		\\
		~=~&
		\frac{2}{\eta} \sum_{i=1}^{n}{R_i \sqrt{\gamma^2 + \sum_{t=1}^{T-1}{\del{g_{t,i} - \tilde{g}_{t,i}}^2}}}
	\end{align*}
	completes the proof.
\end{proof}

The regret bound obtained in the above corollary is much tighter than that of \cite{duchi2011adaptive} and \cite{chiang2012online} when the sequence of loss functions are sparse and predictable. Since we are using per-coordinate learning rates implicitly we get better bounds for the case where only certain coordinates of the gradients are accurately predictable as well. Even when the loss sequence is completely unpredictable, the above bound is not much worse than a constant factor of the bound in \cite{duchi2011adaptive}.  

By using Proposition~\ref{simple-form-omd} we can derive explicit forms of the update rules given by \eqref{ada-omd-opt-eq1} and \eqref{ada-omd-opt-eq2} with regularizers constructed in Corollary~\ref{diagonal-matrix-regret}. For $y_{t+1} = x_{t} - \sqrt{2} Q_{1:t}^{-1} g_t$ and $\hat{y}_{t+1} = x_{t+1} - \sqrt{2} Q_{1:t+1}^{-1} \tilde{g}_{t+1}$, the updates \eqref{ada-omd-opt-eq1} and \eqref{ada-omd-opt-eq2} can be given as $x_{t+1} = \minimize \frac{1}{2} \norm{x - y_{t+1}}_{Q_{1:t}}^2$ and $\hat{x}_{t+1} = \minimize \frac{1}{2} \norm{x - \hat{y}_{t+1}}_{Q_{1:t+1}}^2$ respectively.

The next corollary explains a regularizer construction method with full matrix learning rates, which is an extension of Corollary~\ref{diagonal-matrix-regret}. But this approach is computationally not preferable, especially in high dimensions, as it costs $O\del{n^2}$ per round of operations.
\begin{corollary}
	\label{full-matrix-regret}
	Define $D := \maximizepairvalue{\norm{x-y}_2}$. Let $Q_{1:t} = \del*{G_{1:t}}^{\half}$, where $G_t$'s are given by \eqref{ada-mat-g0},\eqref{ada-mat-g1} and \eqref{ada-mat-gt}. If $r_t$'s are given by $r_0 \del*{x} = 0$ and $r_t \del*{x} = \frac{1}{\sqrt{2}D} \norm{x}_{Q_t}^2$, then the regret of Algorithm~\ref{algo:ada-opt-omd} w.r.t. any $x^* \in \Omega$ is bounded by 
	\begin{equation*}
		\sum_{t=1}^{T}{f_t \del{\hat{x}_t} - f_t \del{x^*}} \leq \sqrt{2} D \, \text{tr} \del*{Q_{1:T}}.
	\end{equation*}
\end{corollary}

\section{Optimistic Mirror Descent with Curved Losses}
\label{improve-curvature}
The following theorem provides a regret bound of Algorithm~\ref{algo:ada-opt-omd} for the case where $f_t$ is $H_t$-strongly convex with respect to some general norm $\norm{\cdot}$. Since this theorem is an extension of Theorem~2.1 from \cite{hazan2007adaptive} for the Optimistic Mirror Descent, this inherits the properties mentioned there such as : $r_t$'s can be chosen without the knowledge of uniform lower bound on $H_t$'s, and $O \del{\log{T}}$ bound can be achieved even when some $H_t \leq 0$ as long as $\frac{H_{1:t}}{t} > 0$. 

\begin{theorem}
	\label{ada-opt-omd-strong-cvx-theorem}
	Let $f_t$ is $H_t$-strongly convex w.r.t. $\norm{\cdot}$ and $H_t \leq \gamma$ for all $t \in \sbr{T}$. If $r_t$'s are given by $r_0 \del{x}=0$, $r_1 \del{x}=\frac{\gamma}{4} \norm{x}^2$, and $r_t \del{x}=\frac{H_{t-1}}{4} \norm{x}^2$ for all $t \geq 2$,	then the regret of Algorithm~\ref{algo:ada-opt-omd} w.r.t. any $x^* \in \Omega$ is bounded by
	\begin{equation*}
		\sum_{t=1}^{T}{f_t \del{\hat{x}_t} - f_t \del{x^*}} ~\leq~ 3 \sum_{t=1}^{T}{\frac{\norm{g_t - \tilde{g}_t}_*^2}{H_{1:t}}} + \frac{\gamma}{4} \norm{x^*-x_1}^2.
	\end{equation*}
\end{theorem}
\begin{proof}
	For the given choice of regularizers, we have $r_{0:t}\del{x}=\frac{H_{1:t-1} + \gamma}{4} \norm{x}^2$ and $\bregmanadaregcumalative{x}{y}{t} = \frac{H_{1:t-1} + \gamma}{4} \norm{x - y}^2$. Since $r_{0:t}$ is $1$-strongly convex w.r.t. $\sqrt{\frac{H_{1:t-1} + \gamma}{2}} \norm{\cdot}$, we have $\norm{\cdot}_{\del{t}} = \sqrt{\frac{H_{1:t-1} + \gamma}{2}} \norm{\cdot}$ and $\norm{\cdot}_{\del{t},*} = \sqrt{\frac{2}{H_{1:t-1} + \gamma}} \norm{\cdot}_*$. Thus for any $x^* \in \Omega$ we have
	\begin{align*}
		& f_t \del{\hat{x}_t} - f_t \del{x^*}
		\\
		~\leq~&
		\tuple{g_t,\hat{x}_t-x^*} - \frac{H_t}{2} \norm{\hat{x}_t - x^*}^2
		\\
		~\leq~&
		\half \norm{g_t - \tilde{g}_t}_{\del{t},*}^2 + \bregmanadareg{x^*}{x_t} - \bregmanadareg{x^*}{x_{t+1}} 
		- \frac{H_t}{2} \norm{\hat{x}_t - x^*}^2
		\\
		~=~&
		\frac{\norm{g_t - \tilde{g}_t}_*^2}{H_{1:t-1} + \gamma} + \frac{H_{1:t-1} + \gamma}{4} \norm{x^* - x_t}^2 
		- \frac{H_{1:t-1} + \gamma}{4} \norm{x^* - x_{t+1}}^2 - \frac{H_t}{2} \norm{\hat{x}_t - x^*}^2,
	\end{align*}
	where the first inequality is due to the strong convexity of $f_t$, and the second inequality is due to Lemma~\ref{ada-opt-omd-linear-regret-lemma}. Observe that
	\begin{align*}
		&
		\sum_{t=1}^{T}{\frac{H_{1:t-1} + \gamma}{4} \setdel*{\norm{x^* - x_t}^2 - \norm{x^* - x_{t+1}}^2}}
		\\
		~=~&
		\sum_{t=1}^{T}{\norm{x^* - x_{t+1}}^2 \setdel*{\frac{H_{1:t} + \gamma}{4} - \frac{H_{1:t-1} + \gamma}{4}}} 
		+ \frac{\gamma}{4} \norm{x^* - x_1}^2 - \frac{H_{1:T} + \gamma}{4} \norm{x^* - x_{T+1}}^2
		\\
		~\leq~&
		\sum_{t=1}^{T}{\frac{H_t}{4} \norm{x^* - x_{t+1}}^2} + \frac{\gamma}{4} \norm{x^* - x_1}^2,
	\end{align*}
	and 
	\begin{align*}
		&
		\sum_{t=1}^{T}{\frac{H_t}{4} \norm{x^* - x_{t+1}}^2 - \frac{H_t}{2} \norm{\hat{x}_t - x^*}^2}  
		\\
		~=~&
		\sum_{t=1}^{T}{\frac{H_t}{4} \setdel*{\norm{x^* - \hat{x}_t + \hat{x}_t - x_{t+1}}^2 - 2 \norm{x^* - \hat{x}_t}^2}}
		\\
		~\leq~&
		\sum_{t=1}^{T}{\frac{H_t}{2} \norm{\hat{x}_t - x_{t+1}}^2} 
		\\
		~\leq~&
		\sum_{t=1}^{T}{\frac{H_{1:t-1}+\gamma}{2} \norm{\hat{x}_t - x_{t+1}}^2} 
		\\
		~\leq~&
		2 \sum_{t=1}^{T}{\frac{\norm{g_t - \tilde{g}_t}_*^2}{H_{1:t-1}+\gamma}},
	\end{align*}
	where the first inequality is obtained by applying the triangular inequality of norms the fact that $(a+b)^2 \leq 2 a^2 + 2 b^2$, the second inequality is due to the facts that $H_t \leq \gamma$ and $H_{1:t-1} \geq 0$, and the third inequality is due to Lemma~\ref{point-loss-connection}.
	
	Now by summing up the instantaneous regrets and using the above observation we get
	\begin{align*}
		\sum_{t=1}^{T}{f_t \del{\hat{x}_t} - f_t \del{x^*}}
		~\leq~&
		3 \sum_{t=1}^{T}{\frac{\norm{g_t - \tilde{g}_t}_*^2}{H_{1:t-1} + \gamma}} + \frac{\gamma}{4} \norm{x^*-x_1}^2 
		\\
		~\leq~&
		3 \sum_{t=1}^{T}{\frac{\norm{g_t - \tilde{g}_t}_*^2}{H_{1:t}}} + \frac{\gamma}{4} \norm{x^*-x_1}^2, 
	\end{align*}
	where the last inequality is due to the fact that $H_t \leq \gamma$.
\end{proof}

In the above theorem if $H_t \geq H > 0$ and $\norm{g_t-\tilde{g}_t}_* \leq 1$ (w.l.o.g) for all $t$, then it obtain a regret bound of the form $O \del*{\log{\sum_{t=1}^{T}{\norm{g_t-\tilde{g}_t}_*^2}}}$. When $H$ is small, however, this guaranteed regret can still be large.

Now instead of running Algorithm~\ref{algo:ada-opt-omd} on the observed sequence of $f_t$'s, we use the modified sequence of loss functions of the form
\begin{equation}
	\label{prox-modification}
	\tilde{f}_t \del{x} := f_t \del{x} + \frac{\lambda_t}{2} \norm{x - \hat{x}_t}^2, \, \lambda_t \geq 0,
\end{equation}
which is already considered in \cite{do2009proximal} for the non-optimistic mirror descent case. Given $f_t$ is $H_t$-strongly convex with respect to $\norm{\cdot}$, $\tilde{f}_t$ is $\del{H_t + \lambda_t}$-strongly convex. Also note that $\partial \tilde{f}_t \del{\hat{x}_t} = \partial f_t \del{\hat{x}_t}$ because the gradient of $\norm{x - \hat{x}_t}^2$ is $0$ when evaluated at $\hat{x}_t$ \cite{do2009proximal}. Thus in the updates \eqref{ada-omd-opt-eq1} and \eqref{ada-omd-opt-eq2} the terms $g_t$ and $\tilde{g}_{t+1}$ remain unchanged, only the regularizers $r_t$'s will change appropriately. By applying Theorem~\ref{ada-opt-omd-strong-cvx-theorem} for the modified sequence of losses given by \eqref{prox-modification} we obtain the following corollary. 

\begin{corollary}
	\label{prox-modification-corr}
	Let $2R = \maximizepairvalue{\norm{x-y}}$. Also let $f_t$ be $H_t$-strongly convex w.r.t. $\norm{\cdot}$, $H_t \leq \gamma$, and $\lambda_t \leq \delta$, for all $t \in \sbr{T}$. If Algorithm~\ref{algo:ada-opt-omd} is performed on the modified functions $\tilde{f}_t$'s with the regularizers $r_t$'s given by $r_0 \del{x}=0$, $r_1 \del{x}=\frac{\gamma + \delta}{4} \norm{x}^2$, and $r_t \del{x}=\frac{H_{t-1} + \lambda_{t-1}}{4} \norm{x}^2$ for all $t \geq 2$, then for any sequence $\lambda_1,...,\lambda_T \geq 0$, we get 
	\begin{equation*}
		\sum_{t=1}^{T}{f_t \del{\hat{x}_t} - f_t \del{x^*}} ~\leq~ 2R^2 \lambda_{1:T} 
		+ 3 \sum_{t=1}^{T}{\frac{\norm{g_t - \tilde{g}_t}_*^2}{\stronghlambda}} + \frac{\gamma + \delta}{4} \norm{x^*-x_1}^2.
	\end{equation*}
\end{corollary}

\begin{algorithm}[tb]
	\caption{Curvature Adaptive and Optimistic Mirror Descent}
	\label{algo:improved-ada-opt-omd}
	\begin{algorithmic}
		\STATE {\bfseries Input:} $r_0 \del{x}=0$ and $r_1 \del{x}=\frac{\gamma + \delta}{4} \norm{x}^2$.
		\STATE {\bfseries Initialize:} $x_1, \hat{x}_1 = 0 \in \Omega$.
		\FOR{$t=1$ {\bfseries to} $T$}
		\STATE Predict $\hat{x}_t$, observe $f_t$, and incur loss $f_t \del{\hat{x}_t}$. \STATE Compute $g_t \in \partial f_t \del{\hat{x}_t}$ and $\tilde{g}_{t+1} \del{g_1,...,g_t}$.
		\STATE Compute $\lambda_t = \frac{\sqrt{\del{\stronghlambdaminus}^2+\frac{6 \norm{g_t - \tilde{g}_t}_*^2}{R^2}}-\del{\stronghlambdaminus}}{2}$
		\STATE Define $r_{t+1} \del{x}=\frac{H_{t} + \lambda_{t}}{4} \norm{x}^2$.
		\STATE Update
		\begin{align*}
			x_{t+1}
			~=~&
			\minimize{\tuple{g_t,x} + \bregmanadareg{x}{x_t}},
			\\
			\hat{x}_{t+1}
			~=~&
			\minimize{\tuple{\tilde{g}_{t+1},x} + \bregmanadaregcumalative{x}{x_{t+1}}{t+1}}.
		\end{align*}
		\ENDFOR
	\end{algorithmic}
\end{algorithm}

In the above corollary if we consider the two terms that depend on $\lambda_t$'s, the first term increases and the second term deceases with the increase of $\lambda_t$'s. Based on the online balancing heuristic approach \cite{hazan2007adaptive}, the positive solution of $2R^2 \lambda_{1:t} = 3 \frac{\norm{g_t - \tilde{g}_t}_*^2}{\stronghlambda}$ is given by 
\begin{equation*}
	\lambda_t = \frac{\sqrt{\del{\stronghlambdaminus}^2+\frac{6 \norm{g_t - \tilde{g}_t}_*^2}{R^2}}-\del{\stronghlambdaminus}}{2}.
\end{equation*}
The resulting algorithm with the above choice of $\lambda_t$ is given in Algorithm~\ref{algo:improved-ada-opt-omd}. By using the Lemma~3.1 from \cite{hazan2007adaptive} we obtain the following regret bound for Algorithm~\ref{algo:improved-ada-opt-omd}.

\begin{theorem}
	\label{curvature-adaptive-theorem}
	The regret of Algorithm~\ref{algo:improved-ada-opt-omd} on the sequence of $f_t$'s with curvature $H_t \geq 0$ is bounded by
	\begin{equation*}
		\sum_{t=1}^{T}{f_t \del{\hat{x}_t} - f_t \del{x^*}} ~\leq~ \frac{\gamma + \delta}{4} \norm{x^*-x_1}^2 
		+ 2 \inf_{\lambda_1^*,...,\lambda_T^*}{\setdel*{2R^2 \lambda_{1:T}^* + 3 \sum_{t=1}^{T}{\frac{\norm{g_t - \tilde{g}_t}_*^2}{H_{1:t}+\lambda_{1:t}^*}}}}.
	\end{equation*}
\end{theorem}
Thus the Algorithm~\ref{algo:improved-ada-opt-omd} achieves a regret bound which is competitive with the bound achievable by the best offline choice of parameters $\lambda_t$'s. From the above theorem we obtain the following two corollaries which show that Algorithm~\ref{algo:improved-ada-opt-omd} achieves intermediate rates between $O \del*{\sqrt{\sum_{t=1}^{T}{\norm{g_t-\tilde{g}_t}_*^2}}}$ and $O \del*{\log{\sum_{t=1}^{T}{\norm{g_t-\tilde{g}_t}_*^2}}}$ depending on the curvature of the losses. 

\begin{corollary}
	For any sequence of convex loss functions $f_t$'s, the bound on the regret of Algorithm~\ref{algo:improved-ada-opt-omd} is $O \del*{\sqrt{\sum_{t=1}^{T}{\norm{g_t-\tilde{g}_t}_*^2}}}$.
\end{corollary}
\begin{proof}
	Let $\lambda_{1}^* = \sqrt{\sum_{t=1}^{T}{\norm{g_t-\tilde{g}_t}_*^2}}$, and $\lambda_{t}^* = 0$ for all $t > 1$.
	\begin{align*}
		&
		2R^2 \lambda_{1:T}^* + 3 \sum_{t=1}^{T}{\frac{\norm{g_t - \tilde{g}_t}_*^2}{H_{1:t}+\lambda_{1:t}^*}} 
		\\
		~=~& 
		2R^2 \sqrt{\sum_{t=1}^{T}{\norm{g_t-\tilde{g}_t}_*^2}} + 3 \sum_{t=1}^{T}{\frac{\norm{g_t - \tilde{g}_t}_*^2}{0+\sqrt{\sum_{t=1}^{T}{\norm{g_t-\tilde{g}_t}_*^2}}}}
		\\
		~=~&
		\del*{2R^2 + 3} \sqrt{\sum_{t=1}^{T}{\norm{g_t-\tilde{g}_t}_2^2}}.
	\end{align*} 
\end{proof}

\begin{corollary}
	Suppose $\norm{g_t-\tilde{g}_t}_* \leq 1$ (w.l.o.g) and $H_t \geq H > 0$ for all $t \in \sbr{T}$. Then the bound on the regret of Algorithm~\ref{algo:improved-ada-opt-omd} is $O \del*{\log{\sum_{t=1}^{T}{\norm{g_t-\tilde{g}_t}_*^2}}}$.
\end{corollary}
\begin{proof}
	Set $\lambda_t^* = 0$ for all $t$.
	\begin{align*}
		2R^2 \lambda_{1:T}^* + 3 \sum_{t=1}^{T}{\frac{\norm{g_t - \tilde{g}_t}_*^2}{H_{1:t}+\lambda_{1:t}^*}} 
		~=~& 
		0 + 3 \sum_{t=1}^{T}{\frac{\norm{g_t - \tilde{g}_t}_*^2}{Ht+0}} 
		\\
		~=~& O \del*{\log{\sum_{t=1}^{T}{\norm{g_t-\tilde{g}_t}_*^2}}},
	\end{align*} 
	where the last inequality is due to the fact that if $a_t \leq 1$ for all $t \in \sbr{T}$, then $\sum_{t=1}^{T}{\frac{a_t}{t}} \leq O \del*{\log{\sum_{t=1}^{T}{a_t}}}$.
\end{proof}

The results obtained here can be extended to the applications discussed in \cite{do2009proximal,orabona2010online} to obtain much tighter results.

\section{Composite Losses}
\label{extensions}
Here we consider the case when observed loss function $f_t$ is composed with some non-negative (possibly non-smooth) convex regularizer term $\psi_t$ to impose certain constraints on the original problem. In this case we generally do not want to linearize the additional regularizer term, thus in the update rules given by \eqref{ada-omd-opt-eq1} and \eqref{ada-omd-opt-eq2} we include $\psi_t$ and $\psi_{t+1}$ respectively without linearizing them. This extension is presented in Algorithm~\ref{algo:comp-ada-opt-omd}. 

\begin{algorithm}[tb]
	\caption{Adaptive and Optimistic Mirror Descent with Composite Losses}
	\label{algo:comp-ada-opt-omd}
	\begin{algorithmic}
		\STATE {\bfseries Input:} regularizers $r_0, r_1 \geq 0$, composite losses $\setdel*{\psi_t}_t$ where $\psi_t \geq 0$.
		\STATE {\bfseries Initialize:} $x_1, \hat{x}_1 = 0 \in \Omega$.
		\FOR{$t=1$ {\bfseries to} $T$}
		\STATE Predict $\hat{x}_t$, observe $f_t$, and incur loss $f_t \del{\hat{x}_t} + \psi_t \del{\hat{x}_t}$. 
		\STATE Compute $g_t \in \partial f_t \del{\hat{x}_t}$ and $\tilde{g}_{t+1} \del{g_1,...,g_t}$.
		\STATE Construct $r_{t+1}$ s.t. $r_{0:t+1}$ is $1$-strongly convex w.r.t. $\norm{\cdot}_{\del{t+1}}$.
		\STATE Update
		\begin{align}
			x_{t+1}
			~=~&
			\minimize{\tuple{g_t,x} + \psi_t \del{x} + \bregmanadareg{x}{x_t}},
			\label{comp-ada-omd-opt-eq1}
			\\
			\hat{x}_{t+1}
			~=~&
			\minimize{\tuple{\tilde{g}_{t+1},x} +\psi_{t+1} \del{x} + \bregmanadaregcumalative{x}{x_{t+1}}{t+1}}.
			\label{comp-ada-omd-opt-eq2}
		\end{align}
		\ENDFOR
	\end{algorithmic}
\end{algorithm}

The following lemma provides a bound on the instantaneous regret of Algorithm~\ref{algo:comp-ada-opt-omd}. 
\begin{lemma}
	\label{comp-ada-omd-optimistic-lemma}
	The instantaneous regret of Algorithm~\ref{algo:comp-ada-opt-omd} w.r.t. any $x^* \in \Omega$ can be bounded as follows
	\begin{equation*}
		\setdel*{\compositefunc{\hat{x}_t}{\hat{x}_t}} - \setdel*{\compositefunc{x^*}{x^*}} 
		~\leq~ \half \norm{g_t-\tilde{g}_t}_{\del{t},*}^2 + \bregmanadareg{x^*}{x_t} - \bregmanadareg{x^*}{x_{t+1}}.
	\end{equation*}
\end{lemma}
\begin{proof}
	The instantaneous regret of the algorithm can be bounded as below using the convexity of $f_t$
	\begin{equation*}
		\setdel*{\compositefunc{\hat{x}_t}{\hat{x}_t}} - \setdel*{\compositefunc{x^*}{x^*}} 
		~\leq~ \tuple{g_t,\hat{x}_t-x^*} + \setdel*{\Psi_t \del{\hat{x}_t} - \Psi_t \del{x^*}}.
	\end{equation*}
	Now consider
	\begin{equation}
		\label{comp-ada-omd-optimistic-proof1}
		\tuple{g_t,\hat{x}_t-x^*} ~=~ \tuple{g_t - \tilde{g}_t,\hat{x_t}-x_{t+1}} + \tuple{\tilde{g}_t,\hat{x_t}-x_{t+1}} 
		+ \tuple{g_t,x_{t+1}-x^*}.
	\end{equation}
	By the fact that $\tuple{a,b} \leq \norm{a} \norm{b}_* \leq \half \norm{a}^2 + \half \norm{b}_*^2$, we have 
	\begin{equation*}
		\tuple{g_t - \tilde{g}_t,\hat{x_t}-x_{t+1}} ~\leq~ \half \norm{\hat{x_t}-x_{t+1}}_{\del{t}}^2 + \half \norm{g_t - \tilde{g}_t}_{\del{t},*}^2
	\end{equation*}
	The first-order optimality condition for $x^* = \minimize{\tuple{g,x} + f \del{x} + \bregmanadaregpsi{x}{y}}$ and for $z \in \Omega$, 
	\begin{equation*}
		\tuple{x^* - z , g} ~\leq~ \tuple{z - x^* , f' \del{x^*}} 
		+ \bregmanadaregpsi{z}{y} - \bregmanadaregpsi{z}{x^*} - \bregmanadaregpsi{x^*}{y}.
	\end{equation*}
	By applying the above condition for \eqref{comp-ada-omd-opt-eq2} and \eqref{comp-ada-omd-opt-eq1} we have respectively 
	\begin{equation*}
		\tuple{\hat{x_t} - x_{t+1} , \tilde{g}_t} ~\leq~ \tuple{\Psi_t' \del{\hat{x_t}} , x_{t+1} - \hat{x_t}} 
		+ \bregmanadareg{x_{t+1}}{x_{t}} - \bregmanadareg{x_{t+1}}{\hat{x_t}} - \bregmanadareg{\hat{x_t}}{x_{t}}
	\end{equation*}
	\begin{equation*}
		\tuple{x_{t+1} - x^* , g_t} ~\leq~ \tuple{\Psi_t' \del{x_{t+1}} , x^* - x_{t+1}} 
		+ \bregmanadareg{x^*}{x_t} - \bregmanadareg{x^*}{x_{t+1}} - \bregmanadareg{x_{t+1}}{x_t}.
	\end{equation*}
	Thus by \eqref{comp-ada-omd-optimistic-proof1} we have
	\begin{align*}
		&\tuple{g_t,\hat{x}_t-x^*} + \setdel*{\Psi_t \del{\hat{x}_t} - \Psi_t \del{x^*}}
		\\
		~\leq~&
		\half \norm{\hat{x_t}-x_{t+1}}_{\del{t}}^2 + \half \norm{g_t - \tilde{g}_t}_{\del{t},*}^2 
		+ \bregmanadareg{x^*}{x_t} - \bregmanadareg{x^*}{x_{t+1}} - \bregmanadareg{x_{t+1}}{\hat{x_t}} 
		\\
		&
		+ \Psi_t \del{\hat{x}_t} - \Psi_t \del{x^*} + \tuple{\Psi_t' \del{\hat{x_t}} , x_{t+1} - \hat{x_t}} 
		+ \tuple{\Psi_t' \del{x_{t+1}} , x^* - x_{t+1}}
		\\
		~\leq~&
		\half \norm{g_t - \tilde{g}_t}_{\del{t},*}^2 + \bregmanadareg{x^*}{x_t} - \bregmanadareg{x^*}{x_{t+1}} 
		\\
		&
		+ \Psi_t \del{\hat{x}_t} + \tuple{\Psi_t' \del{\hat{x_t}} , x_{t+1} - \hat{x_t}} 
		+ \tuple{\Psi_t' \del{x_{t+1}} , x^* - x_{t+1}} - \Psi_t \del{x^*}
		\\
		~\leq~&
		\half \norm{g_t - \tilde{g}_t}_{\del{t},*}^2 + \bregmanadareg{x^*}{x_t} - \bregmanadareg{x^*}{x_{t+1}} 
		\\
		&
		+ \Psi_t \del{x_{t+1}} + \tuple{\Psi_t' \del{x_{t+1}} , x^* - x_{t+1}} - \Psi_t \del{x^*}
		\\
		~\leq~&
		\half \norm{g_t - \tilde{g}_t}_{\del{t},*}^2 + \bregmanadareg{x^*}{x_t} - \bregmanadareg{x^*}{x_{t+1}} 
		+ \Psi_t \del{x^*} - \Psi_t \del{x^*}
		\\
		~=~&
		\half \norm{g_t - \tilde{g}_t}_{\del{t},*}^2 + \bregmanadareg{x^*}{x_t} - \bregmanadareg{x^*}{x_{t+1}}
	\end{align*}
	where the second inequality is due to 1-strong convexity of $r_{0:t}$ w.r.t. $\norm{\cdot}_{\del{t}}$, and the third and fourth inequalities are due to the convexity of $\Psi_t$ at $\hat{x}_t$ and $x_{t+1}$ respectively.
\end{proof}

From the above lemma we can observe that the instantaneous regret of Algorithm~\ref{algo:comp-ada-opt-omd} is exactly equal to that of the non-composite version (Algorithm~\ref{algo:ada-opt-omd}). Thus all the improvements that we discussed in the previous sections for the non-composite case are also applicable to composite losses as well.

\section{Discussion}

We present adaptive variants of optimistic mirror descent which improve the existing regret bounds when some of the \textit{easy data} instances occur together. Algorithms that we have discussed in this paper achieve regret guarantees that hold at any time.

We also note that the regret bounds given in this work can be converted into convergence bounds for batch stochastic problems using online-to-batch conversion techniques \cite{cesa2004generalization,kakade2009generalization}.

As in Theorem~\ref{ada-opt-omd-strong-cvx-theorem}, we can obtain regret bound for the case when the loss function $f_t$ is $\beta_t$-convex (which is broader class than exp-concave losses) as well. But for the resulting bound we cannot apply Lemma~3.1 from \cite{hazan2007adaptive} to obtain a near optimal closed form solution of $\lambda_t$. 

\bibliographystyle{plain}
\bibliography{refs}


\appendix

\section{Proofs}
\label{sec:proof}
\begin{proof} \textbf{(Proposition~\ref{simple-form-omd})}
	Observe that
	\begin{align*}
		&
		\minimize{\bregmanadaregpsi{x}{y}}  
		\\
		~=~&
		\minimize{\psi \del{x} - \psi \del{y} - \tuple{\nabla \psi \del{y}, x - y}} 
		\\
		~=~&
		\minimize{\psi \del{x} - \tuple{\nabla \psi \del{y}, x}}
		\\
		~=~&
		\minimize{\psi \del{x} - \tuple{\nabla \psi \del{u} - g, x}}
		\\
		~=~&
		\minimize{\tuple{g, x} + \psi \del{x} - \psi \del{u} - \tuple{\nabla \psi \del{u}, x - u}}
		\\
		~=~&
		\minimize{\tuple{g, x} + \bregmanadaregpsi{x}{u}}.
	\end{align*}
\end{proof}

\begin{proof} \textbf{(Lemma~\ref{point-loss-connection})}
	Since $r_{0:t}$ is $1$-strongly convex w.r.t. $\norm{\cdot}_{\del{t}}$ we have
	\begin{align*}
		&
		\bregmanadareg{\hat{x}_t}{x_{t+1}} 
		\\
		~=~&
		r_{0:t} \del{\hat{x}_t} - r_{0:t} \del{x_{t+1}} - \tuple{\regderivative{x_{t+1}}, \hat{x}_t - x_{t+1}}
		\\
		~\geq~&
		\half \norm{\hat{x}_t - x_{t+1}}_{\del{t}}^2,
		\\
		\text{and}
		\\
		&
		\bregmanadareg{x_{t+1}}{\hat{x}_{t}}
		\\
		~=~&
		r_{0:t} \del{x_{t+1}} - r_{0:t} \del{\hat{x}_{t}} - \tuple{\regderivative{\hat{x}_{t}}, x_{t+1} - \hat{x}_{t}}
		\\
		~\geq~&
		\half \norm{x_{t+1} - \hat{x}_{t}}_{\del{t}}^2.
	\end{align*}
	Adding these two bounds, we obtain 
	\begin{equation}
		\label{point-loss-connection-proof1}
		\norm{\hat{x}_{t} - x_{t+1}}_{\del{t}}^2 ~\leq~ \tuple{\regderivative{\hat{x}_t} - \regderivative{x_{t+1}} , \hat{x}_{t} - x_{t+1}}.
	\end{equation}
	
	Suppose $y_{t+1}$ and $\hat{y}_{t}$ satisfy the conditions $\regderivative{y_{t+1}} = \regderivative{x_{t}} - g_t$ and $\regderivative{\hat{y}_{t}} = \regderivative{x_{t}} - \tilde{g}_{t}$ respectively. Then by applying Proposition~\ref{simple-form-omd} to the updates in \eqref{ada-omd-opt-eq2} and \eqref{ada-omd-opt-eq1} of Algorithm~\ref{algo:ada-opt-omd}, we obtain
	\begin{align*}
		x_{t+1} 
		~=~& \minimize{\bregmanadareg{x}{y_{t+1}}}
		\\
		\hat{x}_{t} 
		~=~& \minimize{\bregmanadareg{x}{\hat{y}_{t}}}.
	\end{align*}
	By applying the first order optimality condition for the above two optimization statements, we have
	\begin{align*}
		\tuple{\regderivative{x_{t+1}} - \regderivative{y_{t+1}} , \hat{x}_t - x_{t+1}} 
		~\geq~& 0
		\\
		\tuple{\regderivative{\hat{x}_t} - \regderivative{\hat{y}_t} , x_{t+1} - \hat{x}_t}
		~\geq~& 0,
	\end{align*}
	respectively. Combining these two bounds, we obtain
	\begin{equation*}
		\tuple{ \regderivative{\hat{y}_t} - \regderivative{y_{t+1}} , \hat{x}_t - x_{t+1}} 
		~\geq~ \tuple{ \regderivative{\hat{x}_t} - \regderivative{x_{t+1}} , \hat{x}_t - x_{t+1}}.
	\end{equation*}
	By combining the above result with \eqref{point-loss-connection-proof1}, we obtain
	\begin{align*}
		&
		\norm{\hat{x}_{t} - x_{t+1}}_{\del{t}}^2 
		\\
		~\leq~& 
		\tuple{ \regderivative{\hat{y}_t} - \regderivative{y_{t+1}} , \hat{x}_t - x_{t+1}} 
		\\
		~\leq~&
		\norm{\regderivative{\hat{y}_t} - \regderivative{y_{t+1}}}_{\del{t},*} \norm{\hat{x}_t - x_{t+1}}_{\del{t}},
	\end{align*}
	by a generalized Cauchy-Schwartz inequality. Dividing both sides by $\norm{\hat{x}_t - x_{t+1}}_{\del{t}}$, we have
	\begin{align*}
		&
		\norm{\hat{x}_{t} - x_{t+1}}_{\del{t}} 
		\\
		~\leq~&
		\norm{\regderivative{\hat{y}_t} - \regderivative{y_{t+1}}}_{\del{t},*}
		\\
		~=~&
		\norm{\del*{\regderivative{x_{t}} - \tilde{g}_t} - \del*{\regderivative{x_t} - g_t}}_{\del{t},*}
		\\
		~=~&
		\norm{g_t - \tilde{g}_t}_{\del{t},*}.
	\end{align*}
\end{proof}

\begin{proof} \textbf{(Corollary~\ref{full-matrix-regret})}
	By letting $\eta=\frac{D}{\sqrt{2}}$ for the given sequence of regularizers, we get $r_{0:t} \del*{x} = \frac{1}{2 \eta} \norm{x}_{Q_{1:t}}^2$. Since $r_{0:t}$ is $1$-strongly convex w.r.t. $\frac{1}{\sqrt{\eta}} \norm{\cdot}_{Q_{1:t}}$, we have $\norm{\cdot}_{\del*{t}} = \frac{1}{\sqrt{\eta}} \norm{\cdot}_{Q_{1:t}}$ and $\norm{\cdot}_{\del*{t},*} = \sqrt{\eta} \norm{\cdot}_{Q_{1:t}^{-1}}$. By Theorem~\ref{ada-omd-optimistic-thm} the regret bound of Algorithm~\ref{algo:ada-opt-omd} with this choice of regularizer sequence can be given as follows
	\begin{equation*}
		\sum_{t=1}^{T}{f_t \del{\hat{x}_t} - f_t \del{x^*}} ~\leq~ \half \sum_{t=1}^{T}{\norm{g_t - \tilde{g}_{t}}_{\del{t},*}^2} + \sum_{t=1}^{T}{\bregmanadaregsingle{x^*}{x_{t}}{t}}.
	\end{equation*}
	
	Consider
	\begin{align*}
		&\half \sum_{t=1}^{T}{\norm{g_t - \tilde{g}_{t}}_{\del{t},*}^2} 
		\\
		~=~&
		\half \sum_{t=1}^{T}{\eta \norm{g_t - \tilde{g}_{t}}_{Q_{1:t}^{-1}}^2}
		\\
		~=~&
		\frac{\eta}{2} \sum_{t=1}^{T}{\del*{g_t - \tilde{g}_{t}} Q_{1:t}^{-1} \del*{g_t - \tilde{g}_{t}}^T}
		\\
		~=~&
		\frac{\eta}{2} \sum_{t=1}^{T}{\del*{g_t - \tilde{g}_{t}} \del*{\gamma^2 I + G_{2:t}}^{-\half} \del*{g_t - \tilde{g}_{t}}^T}
		\\
		~\leq~&
		\frac{\eta}{2} \sum_{t=1}^{T}{\del*{g_t - \tilde{g}_{t}} \del*{G_{2:t+1}}^{-\half} \del*{g_t - \tilde{g}_{t}}^T}
		\\
		~\leq~&
		\eta \, \text{tr} \del*{G_{2:T+1}^{\half}}
		\\
		~\leq~&
		\eta \, \text{tr} \del*{\del*{\gamma^2 I + G_{2:T}}^{\half}}
		\\
		~=~&
		\eta \, \text{tr} \del*{Q_{1:T}},
	\end{align*}
	where the first inequality is due to the facts that $\gamma^2 I \succcurlyeq G_{t+1}$ and $A \succcurlyeq B \succcurlyeq 0 \Rightarrow A^{\half} \succcurlyeq B^{\half}$ and $B^{-1} \succcurlyeq A^{-1}$, the second inequality is due to the fact that $\sum_{t=1}^{T}{a_t^T \del*{\sum_{s=1}^{t}{a_s a_s^T}}^{-\half} a_t} \leq \text{tr} \del*{\sum_{t=1}^{T}{a_t a_t^T}}$ (see Lemma~10 from \cite{duchi2011adaptive}), and the third inequality is due to the fact that $\gamma^2 I \succcurlyeq G_{T+1}$. Also observing that 
	\begin{align*}
		&\sum_{t=1}^{T}{\bregmanadaregsingle{x^*}{x_{t}}{t}} 
		\\
		~=~&
		\sum_{t=1}^{T}{\frac{1}{2 \eta} \norm{x^* - x_t}_{Q_t}^2}
		\\
		~\leq~&
		\frac{1}{2 \eta} \sum_{t=1}^{T}{\norm{x^* - x_t}_2^2 \lambda_{\text{max}} \del*{Q_t}}
		\\
		~\leq~&
		\frac{1}{2 \eta} \sum_{t=1}^{T}{\norm{x^* - x_t}_2^2 \text{tr} \del*{Q_t}}
		\\
		~\leq~&
		\frac{1}{2 \eta} \sum_{t=1}^{T}{D^2 \text{tr} \del*{Q_t}}
		\\
		~=~&
		\frac{D^2}{2 \eta} \text{tr} \del*{Q_{1:T}}.
	\end{align*}
	completes the proof.
\end{proof}

\section{Mirror Descent with $\beta$-convex losses}
Given a convex set $\Omega \subseteq \mathbb{R}^n$ and $\beta > 0$, a function $f:\Omega \rightarrow \mathbb{R}$ is $\beta$-convex, if for all $x,y \in \Omega$
\begin{equation*}
	f \del{x} ~\geq~ f \del{y} + \tuple{g,x-y} + \beta \norm{x-y}_{gg^T}^2, \, g \in \partial f \del{y}.
\end{equation*}

\begin{algorithm}[tb]
	\caption{Adaptive Mirror Descent}
	\label{algo:ada-omd}
	\begin{algorithmic}
		\STATE {\bfseries Input:} regularizers $r_0 \geq 0$.
		\STATE {\bfseries Initialize:} $x_1 = 0 \in \Omega$.
		\FOR{$t=1$ {\bfseries to} $T$}
		\STATE Predict $x_t$, observe $f_t$, and incur loss $f_t \del{x_t}$. 
		\STATE Compute $g_t \in \partial f_t \del{x_t}$.
		\STATE Construct $r_{t}$ s.t. $r_{0:t}$ is $1$-strongly convex w.r.t. $\norm{\cdot}_{\del{t}}$.
		\STATE Update
		\begin{align}
			\label{ada-omd-eq}
			x_{t+1}
			~=~&
			\minimize{\tuple{g_t,x} + \bregmanadareg{x}{x_t}}. 
		\end{align}
		\ENDFOR
	\end{algorithmic}
\end{algorithm}

\begin{lemma}
	\label{ada-omd-linear-regret-lemma}
	The instantaneous linear regret of Algorithm~\ref{algo:ada-omd} w.r.t. any $x^* \in \Omega$ can be bounded as follows
	\begin{equation*}
		\tuple{x_t - x^* , g_t} ~\leq~ \bregmanadareg{x^*}{x_t} - \bregmanadareg{x^*}{x_{t+1}} + \half \norm{g_t}_{\del{t},*}^2.
	\end{equation*}
\end{lemma}
\begin{proof}
	By the first-order optimality condition for \eqref{ada-omd-eq} we have, 
	\begin{equation}
		\tuple{x - x_{t+1} , g_t + \regderivative{x_{t+1}} - \regderivative{x_t}} \quad \geq \quad 0
		\label{ada-omd-linear-regret-lemma-proof1}
	\end{equation}
	Consider
	\begin{align*}
		&
		\tuple{x_t - x^* , g_t} 
		\\
		~=~&
		\tuple{x_{t+1} - x^* , g_t} + \tuple{x_t - x_{t+1} , g_t} 
		\\
		~\leq~& 
		\tuple{x^* - x_{t+1} , \regderivative{x_{t+1}} - \regderivative{x_t}} 
		+ \tuple{x_t - x_{t+1} , g_t}  
		\\
		~=~& 
		\bregmanadareg{x^*}{x_t} - \bregmanadareg{x^*}{x_{t+1}} - \bregmanadareg{x_{t+1}}{x_t} 
		+ \tuple{x_t - x_{t+1} , g_t}
		\\
		~\leq~& 
		\bregmanadareg{x^*}{x_t} - \bregmanadareg{x^*}{x_{t+1}} - \bregmanadareg{x_{t+1}}{x_t} 
		+ \half \norm{x_t - x_{t+1}}_{\del{t}}^2 + \half \norm{g_t}_{\del{t},*}^2
		\\
		~\leq~& 
		\bregmanadareg{x^*}{x_t} - \bregmanadareg{x^*}{x_{t+1}} + \half \norm{g_t}_{\del{t},*}^2,
	\end{align*}
	where the first inequality is due to \eqref{ada-omd-linear-regret-lemma-proof1}, the second equality is due to the fact that $\tuple{\psiderivative{a} - \psiderivative{b},c-a} =  \bregmanadaregpsi{c}{b} - \bregmanadaregpsi{c}{a} - \bregmanadaregpsi{a}{b}$, the second inequality is due to the fact that $\tuple{a,b} \leq \norm{a} \norm{b}_* \leq \half \norm{a}^2 + \half \norm{b}_*^2$, and the third inequality is due to the $1$-strong convexity of $r_{0:t}$ w.r.t. $\norm{\cdot}_{\del{t}}$.
\end{proof}

\begin{theorem}
	\label{ada-omd-beta-cvx-theorem}
	Let $f_t$ is $\beta_t$-convex, $\forall t \in \sbr{T}$. If $r_t$'s are given by
	\begin{equation}
		\label{ada-omd-reg-choice-beta-t}
		r_t \del{x}=\norm{x}_{h_t}^2, \text{ where } h_0=I_{n \times n} \text{ and } h_t=\beta_t g_t g_t^T \text{ for } t \geq 1,
	\end{equation}
	then the regret of Algorithm~\ref{algo:ada-omd} w.r.t. any $x^* \in \Omega$ is bounded by
	\begin{equation*}
		\sum_{t=1}^{T}{f_t \del{x_t} - f_t \del{x^*}} ~\leq~ \frac{1}{4} \sum_{t=1}^{T}{\norm{g_t}_{h_{0:t}^{-1}}^2} + \norm{x^* - x_1}_2^2 .
	\end{equation*}
\end{theorem}
\begin{proof}
	For the choice of regularizer sequence $\setdel{r_t}$ given by \eqref{ada-omd-reg-choice-beta-t}, we have $r_{0:t}\del{x}=\norm{x}_{h_{0:t}}^2$ and $\bregmanadareg{x}{y}=\half \del*{\sqrt{2}\norm{x-y}_{h_{0:t}}}^2$. Since $r_{0:t}$ is $1$-strongly convex w.r.t. $\sqrt{2}\norm{\cdot}_{h_{0:t}}$, we have $\norm{\cdot}_{\del{t}}=\sqrt{2}\norm{\cdot}_{h_{0:t}}$ and $\norm{\cdot}_{\del{t},*}=\frac{1}{\sqrt{2}} \norm{\cdot}_{h_{0:t}^{-1}}$.
	
	For any $x^* \in \Omega$
	\begin{align*}
		&
		f_t \del{x_t} - f_t \del{x^*}
		\\
		~\leq~&
		\tuple{g_t,x_t - x^*} - \beta_t \norm{x^* - x_t}_{g_t g_t^T}^2 
		\\
		~\leq~&
		\bregmanadareg{x^*}{x_t} - \bregmanadareg{x^*}{x_{t+1}} 
		+ \half \norm{g_t}_{\del{t},*}^2 - \norm{x^* - x_t}_{\beta_t g_t g_t^T}^2
		\\
		~=~&
		\norm{x^* - x_t}_{h_{0:t}}^2 - \norm{x^* - x_{t+1}}_{h_{0:t}}^2 - \norm{x^* - x_t}_{h_t}^2 
		+ \half \norm{g_t}_{\del{t},*}^2, 
	\end{align*}
	where the first inequality is due to the $\beta_t$-convexity of $f_t \del{\cdot}$, and the second inequality is due to Lemma~\ref{ada-omd-linear-regret-lemma}. By summing all the instantaneous regrets we get
	\begin{align*}
		&
		\sum_{t=1}^{T}{f_t \del{x_t}} - \sum_{t=1}^{T}{f_t \del{x^*}} 
		\\
		~\leq~&
		\sum_{t=1}^{T}{\setdel*{\norm{x^* - x_t}_{h_{0:t}}^2 - \norm{x^* - x_t}_{h_{0:t-1}}^2 -\norm{x^* - x_t}_{h_{t}}^2}} 
		+ \norm{x^* - x_1}_{h_{0}}^2 - \norm{x^* - x_{T+1}}_{h_{0:T}}^2 + \half \sum_{t=1}^{T}{\norm{g_t}_{\del{t},*}^2}
		\\
		~\leq~&
		\norm{x^* - x_1}_2^2 + \frac{1}{4} \sum_{t=1}^{T}{\norm{g_t}_{h_{0:t}^{-1}}^2}.
	\end{align*}
\end{proof}

Now instead of running Algorithm~\ref{algo:ada-omd} on the observed sequence of $f_t$'s, we use the modified sequence of loss functions of the form
\begin{equation}
	\label{prox-modification-beta}
	\tilde{f}_t \del{x} := f_t \del{x} + \lambda_t g\del{x}, \, \lambda_t \geq 0,
\end{equation}
where $g\del{x}$ is $1$-convex. By following the proof of Theorem~\ref{ada-omd-beta-cvx-theorem} for the modified sequence of losses given by \eqref{prox-modification-beta} we obtain the following corollary.

\begin{theorem}
	\label{curve-case-1-beta-thm1}
	Let $g\del{x}$ be a $1$-convex function, $A^2 = \maximizevalue{g \del{x}}$ and $B = \maximizevalue{\norm{g' \del{x}}_{\del*{g' \del{x} g' \del{x}^T}^{-1}}}$. Also let $f_t$ be $\beta_t$-convex ($\beta_t \geq 0$), $\forall t \in \sbr{T}$. 
	If Algorithm~\ref{algo:ada-omd} is performed on the modified functions $\tilde{f}_t$'s with the regularizers $r_t$'s given by 
	\begin{equation}
		\label{curve-case-1-reg-choice-beta-t}
		r_t \del{x}=\norm{x}_{h_t}^2, \text{ where } h_0=I_{n \times n}, \text{ and }
		h_t=\beta_t g_t g_t^T + \lambda_t g' \del{x_t} g' \del{x_t}^T, \text{ for } t \geq 1,
	\end{equation} 
	then for any sequence $\lambda_1,...,\lambda_T \geq 0$, we get
	\begin{equation*}
		\sum_{t=1}^{T}{f_t \del{x_t} - f_t \del{x^*}} ~\leq~ \del*{A^2+\frac{B^2}{2}} \lambda_{1:T}
		+ \half \sum_{t=1}^{T}{\norm{g_t}_{h_{0:t}^{-1}}^2} + \norm{x^* - x_1}_2^2.
	\end{equation*}
\end{theorem}
\begin{proof}
	Since $f_t$ is $\beta_t$-convex and $g$ is $1$-convex, for any $x^* \in \Omega$ we have
	\begin{align*}
		&
		\setdel*{f_t \del{x_t} + \lambda_t g \del{x_t}} - \setdel*{f_t \del{x^*} + \lambda_t g \del{x^*}}
		\\
		~=~&
		f_t \del{x_t} - f_t \del{x^*} + \lambda_t \setdel*{g \del{x_t} - g \del{x^*}}
		\\
		~\leq~&
		\tuple{g_t,x_t - x^*} - \beta_t \norm{x^* - x_t}_{g_t g_t^T}^2 
		+ \lambda_t \setdel*{\tuple{g' \del{x_t},x_t - x^*} - \norm{x^* - x_t}_{g' \del{x_t} g' \del{x_t}^T}^2}
		\\
		~=~&
		\tuple{g_t + \lambda_t g' \del{x_t} , x_t - x^*} - \norm{x^* - x_t}_{\beta_t g_t g_t^T + \lambda_t g' \del{x_t} g' \del{x_t}^T}^2.  
	\end{align*}
	By following the similar steps from the proof of Theorem~\ref{ada-omd-beta-cvx-theorem} we get
	\begin{equation*}
		\sum_{t=1}^{T}{f_t \del{x_t} + \lambda_t g \del{x_t}} - \setdel*{\sum_{t=1}^{T}{f_t \del{x^*} + \lambda_t g \del{x^*}}} 
		~\leq~ 
		\frac{1}{4} \sum_{t=1}^{T}{\norm{g_t + \lambda_t g' \del{x_t}}_{h_{0:t}^{-1}}^2} + \norm{x^* - x_1}_2^2.
	\end{equation*}
	By using the facts that $\norm{x+y}_A^2 \leq 2 \norm{x}_A^2 + 2 \norm{y}_A^2$, $h_{0:t} \succcurlyeq h_t \succcurlyeq \lambda_t g' \del{x_t} g' \del{x_t}^T$, and $\norm{g' \del{x_t}}_{\del*{g' \del{x_t} g' \del{x_t}^T}^{-1}} \leq B$, we have
	\begin{align*}
		&
		\sum_{t=1}^{T}{f_t \del{x_t} + \lambda_t g \del{x_t}} - \setdel*{\sum_{t=1}^{T}{f_t \del{x^*} + \lambda_t g \del{x^*}}} 
		\\
		~\leq~&
		\half \sum_{t=1}^{T}{\setdel*{\norm{g_t}_{h_{0:t}^{-1}}^2 + \lambda_t^2 \norm{g' \del{x_t}}_{h_{0:t}^{-1}}^2}} + \norm{x^* - x_1}_2^2
		\\
		~\leq~&
		\half \sum_{t=1}^{T}{\setdel*{\norm{g_t}_{h_{0:t}^{-1}}^2 + \lambda_t^2 \norm{g' \del{x_t}}_{\del*{\lambda_t g' \del{x_t} g' \del{x_t}^T}^{-1}}^2}} 
		+ \norm{x^* - x_1}_2^2
		\\
		~\leq~&
		\half \sum_{t=1}^{T}{\norm{g_t}_{h_{0:t}^{-1}}^2} + \frac{B^2}{2} \lambda_{1:T} + \norm{x^* - x_1}_2^2.
	\end{align*}
	By neglecting the $g \del{x_t}$ terms in the L.H.S. and using the fact that $g\del{x^*} \leq A^2$ we get
	\begin{equation*}
		\sum_{t=1}^{T}{f_t \del{x_t}} ~\leq~ \sum_{t=1}^{T}{f_t \del{x^*}} + A^2 \lambda_{1:T} + \half \sum_{t=1}^{T}{\norm{g_t}_{h_{0:t}^{-1}}^2} 
		+ \norm{x^* - x_1}_2^2 + \frac{B^2}{2} \lambda_{1:T}.
	\end{equation*}
\end{proof}

But we cannot apply Lemma~3.1 from \cite{hazan2007adaptive} for the above regret bound to obtain a near optimal closed form solution to $\lambda_t$. One could employ an optimization algorithm to find the optimal $\lambda_t$.

\end{document}